\documentclass[conference]{IEEEtran}
\usepackage[noadjust]{cite}

\newcommand\numberthis{\addtocounter{equation}{1}\tag{\theequation}}

\usepackage[ruled,vlined]{algorithm2e}
\usepackage{afterpage}
\usepackage{soul,optidef,verbatim}
\usepackage{bm}
\usepackage{float}
\usepackage[center]{caption}
\usepackage{url}
\usepackage{amsthm}
\usepackage{amsmath,amssymb,amsfonts}
\usepackage{algorithmic}
\usepackage{graphicx}
\usepackage{xcolor}
\newtheorem{definition}{Definition}
\newtheorem{lemma}{Lemma}
\newtheorem{theorem}{Theorem}
\newtheorem{proposition}{Proposition}

\newtheorem{assumption}{Assumption}

\usepackage{cancel}

\title{Federated Learning with Communication \\ Delay in Edge Networks}

\author{\IEEEauthorblockN{Frank Po-Chen Lin,
Christopher G. Brinton, Nicol\`{o} Michelusi}
\IEEEauthorblockA{School of Electrical and Computer Engineering,
Purdue University, USA\\
Email: \{lin1183,cgb,michelus\}@purdue.edu}}

\begin{document}

\setulcolor{red}
\setul{red}{2pt}
\setstcolor{red}

\maketitle

\begin{abstract}
Federated learning has received significant attention as a potential solution for distributing machine learning (ML) model training through edge networks. This work addresses an important consideration of federated learning at the network edge: communication delays between the edge nodes and the aggregator. A technique called {\tt FedDelAvg} (federated delayed averaging) is developed, which generalizes the standard federated averaging algorithm to incorporate a weighting between the current local model and the delayed global model received at each device during the synchronization step. Through theoretical analysis, an upper bound is derived on the global model loss achieved by {\tt FedDelAvg}, which reveals a strong dependency of learning performance on the values of the weighting and learning rate. Experimental results on a popular ML task indicate significant improvements in terms of convergence speed when optimizing the weighting scheme to account for delays.
\end{abstract}


\section{Introduction}
Rapid developments of communications technology in conjunction with machine learning (ML) algorithms has resulted in an exponential rise in data generated by user devices \cite{CVN}. A paradigm shift is occurring from the conventional cloud computing architecture to a hybrid cloud-edge model where ML data processing is carried out on edge devices for many applications, particularly latency-sensitive tasks \cite{Chiang}.

Federated learning (FL) has emerged as a promising solution for distributing the training of ML models across edge devices. FL techniques, such as the popular {\tt FedAvg} algorithm \cite{McMahan}, generally consist of three steps repeated in sequence: (i) several iterations of parallel, local model training at each device using their own local datasets, (ii) aggregation of the local models at an edge server into a single, global model, and (iii) synchronization of the local models at each device with this global model. 

Several research efforts have been conducted on federated learning in recent years to address challenges such as reducing communication overhead and analyzing convergence rates. In this paper, we consider another important aspect of FL that arises in edge networks: communication delays between the edge devices and server performing the aggregations. Our proposed algorithm, {\tt FedDelAvg}, serves the dual purpose of quantifying the impact of such delays and optimizing model performance in their presence.

\subsubsection{Related work}
We divide related work on federated learning into two categories: studies on (i) reducing communication bandwidth requirements and (ii) obtaining model convergence bounds. For a recent, comprehensive survey of works on FL, see \cite{Kairouz}.

\noindent \textbf{Reducing communication requirements.}
Edge networks can have unreliable communication environments (e.g., variable wireless connections), which can impact distributed ML techniques. For this reason, \cite{McMahan,K2} proposed methods to reduce the number of upstream and downstream communication rounds required in FL. Other efforts have focused on reducing communication demand per round; in particular, \cite{Lin,K3,Horvath} proposed gradient compression methods to reduce the bandwidth required in each transmission. Recently, \cite{Tu} proposed a network-aware FL architecture which trades off communication demand with model convergence. 


\noindent \textbf{Model convergence bounds.}
Other work on FL has studied model convergence under different data distributions and local update models. \cite{Khaled} showed that the error bound of FL in the case of non-Independent and Identically Distributed (non-i.i.d.) data samples becomes worse than i.i.d., while \cite{Li} analyzed the convergence of the {\tt FedAvg} algorithm, suggesting conditions on the learning rate to achieve the optimum. \cite{Yu} discovered that training may converge faster with batch gradient descent, assuming that all edge devices participate throughout the training process. Most recently, \cite{Wang} analyzed the convergence bound of FL in the presence of a total network resource budget constraint.

The above-mentioned works do not consider the effect of network delays. In practice, delays between the edge devices and the cloud are non negligible -- usually from hundreds of milliseconds to several seconds depending on the network bandwidth \cite{Lin} -- which might severely degrade the performance of FL schemes. This aspect is the focus of our work.

\subsubsection{Outline of contributions}
We propose {\tt FedDelAvg}, a novel technique that adapts FL in the presence of network delays. Specifically, we develop a new algorithm for the synchronization step that combines local and global models to account for the effects of delay (Section II). Then, we characterize convergence of {\tt FedDelAvg} and provide suggestions for the optimal weight used in the synchronization phase (Section III). Finally, the delay-robustness of FedDelAvg on convergence speed is demonstrated numerically when the synchronization weighting is adjusted for delay (Section IV).

\section{{\tt FedDelAvg}: Federated Delayed Averaging}
In this section, we introduce the federated learning system model, the machine learning task model, and develop {\tt FedDelAvg}, our federated delayed averaging algorithm. 

\subsection{Edge Network Model}
The federated learning (FL) system architecture consists of a single edge server and multiple edge devices indexed by $i = 1,2,...,N$, as shown in Fig.\ref{fig2}. The edge devices collect data and perform local updates to optimize a loss function $F(\cdot)$ corresponding to a machine learning task (described next). The edge server (the cloud) plays the role of an aggregator, collecting the locally trained parameters $\mathbf w_i$ and the corresponding local loss functions $F_i(\mathbf w_i)$ from the edge devices to perform a global update. Local updates are taken to be gradient descent steps on the local loss functions $F_i(\mathbf w)$, while global updates refers to aggregation followed by synchronization. Aggregation denotes the computation of a global model obtained using the weighted average of local models, while synchronization represents the update of local models at the edge after aggregation \cite{Tu}.

In an edge network, the aggregation of the local model parameters at the cloud followed by the synchronization at the edge incurs communication delay, which we aim to model in our formulation.
\begin{figure}[t!]
  \centering
  \vspace{-2em}
    \includegraphics[height=57mm,width=75mm]{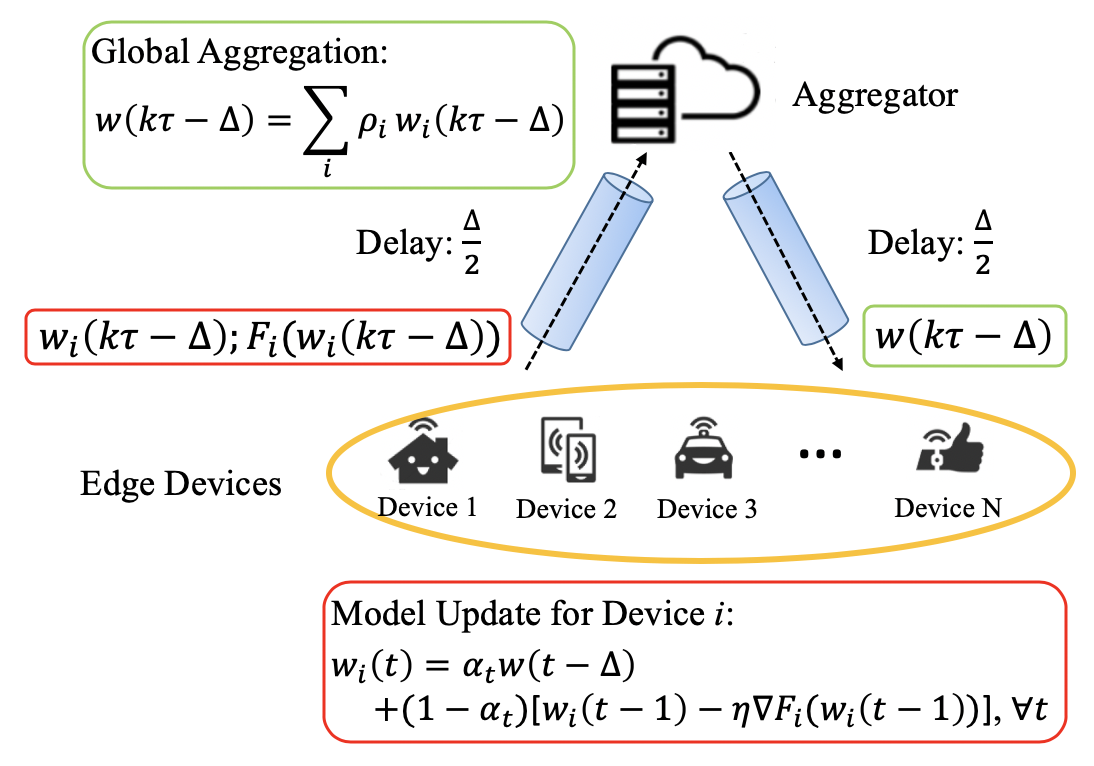}
     \caption{System architecture of {\tt FedDelAvg}. Model synchronization combines global and local models to account for communication delays.}
     \label{fig2}
     \vspace{-0.5em}
\end{figure}

\subsection{Machine Learning Model}
\subsubsection{Data structure}
Each device $i$ carries a dataset $\mathcal{D}_i$ with $D_i=|\mathcal{D}_i|$ data points. Each data point $(\mathbf x,y)\in\mathcal D_i$ consists of an $m$-dimensional feature vector $\mathbf x\in\mathbb R^m$ and a label $y\in\mathbb R$.
   
\subsubsection{Loss function}
We let $f_i(\mathbf x,y;\mathbf w)$ be the loss associated with the 
data point $(\mathbf x,y)\in\mathcal D_i$ based on a model parameter vector $\mathbf w \in \mathbb{R}^m$. For instance, in linear regression, 
the loss function is the squared error $f_i(\mathbf x,y;\mathbf w)=\frac{1}{2}(y-\mathbf w^T\mathbf x)^2$.
We define the loss function across the local dataset $\mathcal D_i$ as
\begin{align}\label{eq:1}
    F_i(\mathbf w)=\frac{1}{D_i}\sum\limits_{(\mathbf x,y)\in\mathcal D_i}
    f_i(\mathbf x,y;\mathbf w),
\end{align}
and the global loss function across all nodes can then be expressed as
\begin{align} \label{eq:2}
    F(\mathbf w)=\sum\limits_{i=1}^{N} \rho_i F_i(\mathbf w),
\end{align}
where $\rho_i = D_i / \sum_{j} D_j$ is the weight associated with the $i$th node, proportional to the size of the local dataset.

\subsubsection{Learning objective}
The goal of the machine learning task is to find the $\mathbf w^*$ that minimizes $F(\mathbf w)$, i.e.,
\begin{align}
    \mathbf w^* = \mathop{\arg\min_{\mathbf w}} F(\mathbf w).
\end{align}

To aid our analysis in Section III, we make a few standard assumptions \cite{McMahan} on the local loss functions $F_i$.
\begin{assumption}
\label{assum1}
$F_i$ is continuously
differentiable, convex, $L$-Lipschitz and $\beta$-smooth, implying that
 \begin{align}
 &| F_i(\mathbf w_1)-F_i(\mathbf w_2)| \leq L\Vert \mathbf w_1- \mathbf w_2\Vert,\\
 &\Vert \nabla F_i(\mathbf w_1)-\nabla F_i(\mathbf w_2)\Vert \leq \beta\Vert \mathbf w_1-\mathbf w_2 \Vert,
 \end{align}
 where $L \in [0, \infty)$ and $\beta \in [0, \infty)$ are the Lipschitz and smoothness constants, respectively.
\end{assumption}

By \eqref{eq:2}, Assumption \ref{assum1} holds for the global loss function $F$ too. One example of $F_i$ that satisfies Assumption \ref{assum1} is the well-known logistic regression loss.

We now state an important assumption on the dissimilarity of data at the edge devices, in addition to Assumption \ref{assum1}:
\begin{assumption} \label{def:2}
The gradients of the local and global loss functions exhibit a similarity of
\begin{align} \label{eq:11}
    \Vert \nabla F_i(\mathbf w)- \nabla F(\mathbf w) \Vert \leq \delta_i,\ \forall \mathbf w,
\end{align}
where $2L\geq\delta_i\geq 0$ is the dissimilarity parameter for node $i$.
\end{assumption} 
We let $\delta=\sum_i \rho_i \delta_i$ be the average data dissimilarity across the edge network.

\subsubsection{Centralized gradient descent}
Loss functions are typically minimized by gradient descent (GD) iterations. In a centralized case, where the global loss $F$ can be optimized directly, this is defined as
\begin{align} \label{eq:3}
    \mathbf w_{GD}(t)=\mathbf w_{GD}(t-1)-\eta\nabla F(\mathbf w_{GD}(t-1)),\
\end{align}
where $\mathbf w_{GD}(0)$ is an initialization,
 $t\geq 1$ is the iteration index, and $\eta>0$ is the learning rate. If $F$ is convex and $\eta \leq \frac{1}{\beta}$, then gradient descent converges to the globally optimal solution $\mathbf w^*$ with rate $\mathcal O(1/T)$, where $T$ is the number of iterations \cite{Bubeck}.

However, centralized gradient descent cannot be directly applied to the FL framework in Fig.\ref{fig2} since no device has direct access to all the data. 
In addition, communication to the cloud is costly in terms of network resources, so the aggregation and synchronization processes are done only periodically.
Finally, communication delay between edge and cloud is usually non-negligible, which we address next in developing the {\tt FedDelAvg} algorithm. 

\subsection{FedDelAvg Algorithm}
In {\tt FedDelAvg}, i.e., Federated Delayed Averaging, the effect of communication delay between edge and cloud on learning performance is incorporated into the design of the FL system. We divide the learning process into discrete time intervals $t\in\{1,2,...,T\}$, where the duration between two consecutive aggregations is denoted as $\tau$. The communication delay between the time when edge devices send their updates to the cloud and the resulting synchronization is denoted $\Delta$, where $\tau\geq\Delta\geq 0$. In Fig.1, we assume a symmetric delay of $\tau / 2$ upstream and downstream. 

\subsubsection{Distributed gradient descent incorporating delay}
We let $\mathbf w_i(t)$ be the local model parameter
vector of edge device $i$ at time $t$, initialized as $\mathbf w_i(-\Delta)=\mathbf w(-\Delta)$ at time $-\Delta$ across all devices.
Let
\begin{align} \label{eq:6}
        \mathbf w(t) = 
           \sum\limits_i \rho_i \mathbf w_i(t)
\end{align}
be the weighted average of the parameter vectors across the edge devices at $t$.
$\mathbf w_i(t)$ is sent at $t\in\{k\tau-\Delta,\forall k\geq 0\}$, $\mathbf w(t)$ is computed after the network delay (e.g. $\Delta/2$) at the cloud, for times $t\in\{k\tau-\Delta/2,\forall k\geq 0\}$ as a result of aggregation, and is then received at the edge devices at time $t+\Delta$.
We partition the time interval $(-\Delta,T-\Delta]$ of duration $T$ into $K=T/\tau$ periods, each of duration $\tau$ (without loss of generality, we assume $T$ is an integer multiple of $\tau$).

Consider the $k$th period, $k\in\{0,\dots, K-1\}$, spanning the time interval $\mathcal T_k = \{k\tau-\Delta+1,...,(k+1)\tau-\Delta\}$.
At each time $t \in \mathcal T_k \setminus\{k\tau\}$, each edge device performs local GD updates as
\begin{align} \label{8}
\mathbf w_i(t) = \mathbf w_i(t-1) - \eta\nabla F_i(\mathbf w_i(t-1)).
\end{align}

At time $t=k\tau$ (synchronization), each edge device 
receives the delayed global parameter vector $\mathbf w(t-\Delta)$
from the cloud. To update the local model, each device first performs a local GD update, followed by a weighted average between the local and global variables; mathematically, at time $t=k\tau$,
\begin{align} \label{9}
\nonumber
\mathbf w_i(t) =&
\alpha \mathbf w(t-\Delta)\\&+(1-\alpha)
\left[\mathbf w_i(t-1) - \eta\nabla F_i(\mathbf w_i(t-1))\right],
\end{align}
where $\alpha\in[0,1]$ is a weight parameter weighting the local vs. global updates. Note that, when $\alpha=1$ and $\Delta=0$, we obtain \cite{McMahan} as a special case.
Letting
\begin{align} 
        \alpha_t = 
        \begin{cases}
            \alpha  ,& t = k\tau,\exists k\in\{0,1,\dots,K-1\} \\
            0 ,& \text{otherwise}
        \end{cases},
    \end{align}
   we can then define the updates at all times $t \in \{-\Delta+1, ..., T-\Delta\}$ as
   \begin{align} \label{update}
   \mathbf w_i(t) =& 
            \alpha_t \sum\limits_j \rho_j \mathbf w_j(t-\Delta)\nonumber
            \\&
+(1-\alpha_t)
\left[\mathbf w_i(t-1) - \eta\nabla F_i(\mathbf w_i(t-1))\right].
\end{align}

\begin{algorithm} 
\small
\SetAlgoLined
\caption{Federated Delayed Averaging} \label{alg}
\KwIn{$\alpha_t,\tau,N,T$}
\KwOut{$\textbf{w}^{K}$}
 Initialize $\mathbf w_i(-\Delta),\  \forall i$\;
 \For{$k=0:K-1$}{
     \For{$t=k\tau-\Delta+1:(k+1)\tau-\Delta$}{
     For each edge device $i\in {1,2,...,N}$ in parallel, update local model with \eqref{update}\;
      \uIf{$t=(k+1)\tau-\Delta$}{
      Each edge device $i$ send local parameters $\mathbf w_i$ and $F_i(\mathbf w_i)$ to the cloud\;
       }
       \uElseIf{$t=(k+1)\tau-\Delta/2$}{
            Compute $\mathbf w((k+1)\tau-\Delta)$ with \eqref{eq:6} and send it to the edge for synchronization\;
            Update $\mathbf w^K$ with \eqref{eq:8}\;
          }
     }
 }
\end{algorithm}

Since edge devices send their local parameters $\mathbf w_i(t)$ and the corresponding local loss functions $F_i(\mathbf w_i(t))$ to the cloud at $t=k\tau-\Delta$, the cloud only has access to the global model $\mathbf w(t)$ and $F(\mathbf w(t))$ at times $t=k\tau-\Delta$. Then, the final model parameter chosen from {\tt FedDelAvg}, after $K$ global aggregations, is
\begin{align} \label{eq:8}
    \mathbf w^K = \mathop{\arg\min_{\mathbf w\in\mathcal W}} F(\mathbf w),
\end{align}
where $\mathcal W\equiv\{\mathbf w(k\tau-\Delta), k={0,1,\dots,K-1}\}$.

The full {\tt FedDelAvg} algorithm is summarized in Alg.\ref{alg}.

\section{Convergence Analysis of {\tt FedDelAvg}}
In this section, we study the convergence of
{\tt FedDelAvg} in terms of the optimality gap
 $F(\mathbf w^K)-F(\mathbf w^*)$ between the
global objective function at the algorithm output $\mathbf w^K$
and at the globally optimal parameter vector $\mathbf w^*$.

\begin{definition} \label{def:1}
For each period $k\in\{0,\dots,K-1\}$, as in \cite{Wang}, we define $\boldsymbol{c}_k(t)$ as the centralized gradient descent algorithm during the time interval $t\in\mathcal T_k$, i.e.,
\begin{align} \label{eq:10}
    \boldsymbol{c}_k(t) = \boldsymbol{c}_k(t-1) - \eta\nabla F(\boldsymbol{c}_k(t-1)),
\end{align}
initialized as $\boldsymbol{c}_k(k\tau-\Delta) = \mathbf w(k\tau-\Delta)$.
\end{definition}
\subsection{Optimality Gap and Optimization}
The main result is demonstrated in Theorem \ref{thm1}, which upper bounds the optimality gap under delay.

\begin{theorem} \label{thm1}
Under Assumption \ref{assum1} and with $\eta < \frac{2}{\beta}$,
\begin{align} \label{78}
    &F(\mathbf w^K)-F(\mathbf w^*)
    \nonumber \\
    &\leq \frac{1}{2\eta\phi T}+\sqrt{\frac{1}{4\eta^2\phi^2 T^2}+ \frac{L\Psi(\alpha)}{\eta\phi T}}+L\psi(\alpha,K),
\end{align}
where 
\begin{align}
&    \Psi(\alpha)\triangleq\sum\limits_{k=1}^{K}\psi(\alpha,k)
    =K\psi(\alpha,\infty)
    \\&
    -[(1+\eta\beta)^{\tau}-1]
    \frac{(1-\alpha)^{2}}{\alpha}
        \epsilon^{(K)}
    ,
    \nonumber
    \\&
\psi(\alpha,k)
    \triangleq
(1-\alpha)\epsilon^{(k)}[(1+\eta\beta)^{\tau}-1]
\\&
+(1-\alpha)h(\tau)+\alpha h(\tau-\Delta)+\alpha\eta\Delta L(1+\eta\beta)^{\tau-\Delta},
\nonumber
\\&
h(x)\triangleq\frac{\delta}{\beta}[(1+\eta\beta)^x -1]-\eta\delta x,
\\&
\epsilon^{(k)}\triangleq
[1-(1-\alpha)^{k}]2\eta L(\tau/\alpha-\Delta).
\end{align}
\end{theorem}

We discuss the proof of Theorem 1 in Section III-B. Theorem \ref{thm1} demonstrates that the performance of {\tt FedDelAvg} under communication delay is strongly dependent on the learning rate $\eta$ and on the value of the weighting $\alpha$ used in the synchronization phase, indicating that these algorithm parameters should be carefully selected with respect to the communication delay in a given learning environment.
One way to design $\alpha$ and $\eta$ is to minimize the asymptotic optimality gap, achieved in the limit $T\to\infty$. In this case, we obtain
 \begin{align}
 \label{optgap}
        &\lim_{K\to\infty}F(\mathbf w^K)-F(\mathbf w^*)\leq \sqrt{\frac{L}{\eta\phi\tau}}\sqrt{\psi(\alpha,\infty)}
        +L\psi(\alpha,\infty),
    \end{align}
    where
    \begin{align*}
    \label{psiinf}
\numberthis
&   \psi(\alpha,\infty)=
(1-\alpha)2\eta L(\tau/\alpha-\Delta)
[(1+\eta\beta)^{\tau}-1]
\\&
+(1-\alpha)h(\tau)+\alpha h(\tau-\Delta)+\alpha\eta\Delta L(1+\eta\beta)^{\tau-\Delta}.
\end{align*}
Note that 
\eqref{optgap} is increasing in $\psi(\alpha,\infty)$. Therefore, the optimal value of $\alpha$ is the minimizer of $\psi(\alpha,\infty)$ in \eqref{psiinf}.
 When the delay is negligible ($\Delta=0$), we obtain
 \begin{align*}
\numberthis
   \psi(\alpha,\infty)&= h(\tau)+
(\alpha^{-1}-1)2\eta L\tau
[(1+\eta\beta)^{\tau}-1]
\\&
\geq \psi(1,\infty),
\end{align*}
and therefore the asymptotic optimality gap in \eqref{optgap} is minimized by choosing $\alpha=1$. In this case, we obtain the {\tt FedAvg} algorithm derived in \cite[Theorem 2]{Wang} as a special case of our analysis. Intuitively, $\alpha=1$ is the optimum for this special case because the global model obtained by the cloud at $t=k\tau$ is built based on the weighted average of up-to-date local models. 

Our analysis generalizes that in \cite[Theorem 2]{Wang} by incorporating communication delay, and by allowing $\alpha\in(0,1]$ in the synchronization phase. In this case, the term $\psi(\alpha,\infty)$ is an increasing function of $\alpha\in(0,1]$ iff
\begin{align*}
&
\alpha^2
\Big\{
\eta\Delta L[2(1+\eta\beta)^{\tau}+(1+\eta\beta)^{\tau-\Delta}-2]
\\&
- \frac{\delta}{\beta}(1+\eta\beta)^{\tau-\Delta}[(1+\eta\beta)^{\Delta}-1]
+\eta\delta\Delta
\Big\}\\&
\geq 
2\eta L\tau[(1+\eta\beta)^{\tau}-1].
\end{align*}
Note that, if 
\begin{align*}
&
\delta
\geq
\eta\beta L
\frac{\Delta(1+\eta\beta)^{\tau-\Delta}-2[(1+\eta\beta)^{\tau}-1](\tau-\Delta)}{
(1+\eta\beta)^{\tau}-(1+\eta\beta)^{\tau-\Delta}-\eta\beta\Delta
}
\end{align*}
then $\psi(\alpha,\infty)$ is a decreasing function of $\alpha$, minimized at $\alpha=1$. This result demonstrates that as data dissimilarity $\delta$ among various edge device larger than a certain threshold, the global parameters across the overall federated system dominates the learning performance such that $\alpha=1$ becomes the optimum. The large dissimilarity reduces the importance of any particular local model since it becomes harder for any of them to truly reflect the characteristics of the overall data. Only by gathering different local models across the network can the learning system form a representative whole for all data participated in the training.
Otherwise, the optimal $\alpha\in(0,1]$ is
\begin{align*}
&
\alpha
=
\sqrt{
\frac{2\eta L\tau[(1+\eta\beta)^{\tau}-1]}
{
\left[\begin{array}{c}
\eta\Delta L[2(1+\eta\beta)^{\tau}+(1+\eta\beta)^{\tau-\Delta}-2]\\
- \frac{\delta}{\beta}(1+\eta\beta)^{\tau-\Delta}[(1+\eta\beta)^{\Delta}-1]
+\eta\delta\Delta
\end{array}\right]
}}.
\end{align*}

In Theorem 1, notice that $F(\mathbf w^K)$ does not converge to the optimum as $T$ increases to infinity. This is due to the fact that the model parameters obtained by {\tt FedDelAvg} with any fixed learning rate $\eta$ will converge to a sub-optimal point. Reference \cite{Li} proved that the decay of learning rate in each training iteration is necessary for {\tt FedAvg} to converge even when assuming the loss function to be  strongly convex and smooth. We leave the algorithm design and convergence analysis for this more general case for future work.

\subsection{Proof of Theorem \ref{thm1}}
In order to prove Theorem \ref{thm1}, we introduce several properties of {\tt FedDelAvg} through supporting lemmas and propositions. The detailed proofs are
provided in Appendix A.

\setcounter{lemma}{2}
\begin{lemma} \label{epsilon}
Under Assumption \ref{assum1}, with $\eta < \frac{2}{\beta}$, we have
\begin{align} \label{eq:34}
    \Vert\mathbf w_i(k\tau-\Delta)-\mathbf w(k\tau-\Delta)\Vert \leq \epsilon^{(k)}
\end{align}
for all $k \in{1,2,...,K}$, where
\begin{align}
\epsilon^{(k)}=
[1-(1-\alpha)^{k}]2\eta L(\tau/\alpha-\Delta).
\end{align}

\end{lemma}


Lemma \ref{epsilon} bounds the error between the local model $\mathbf w_i(t)$ and the global $\mathbf w(t)$ by $\epsilon^{(k)}$ at time $t=k\tau-\Delta$ for all $k$. It can be observed that the difference between $\mathbf w_i(t)$ and $\mathbf w(t)$, depending on $\alpha$, increases as the training process continues. However, the rate of increase in this difference continues to decrease until $\epsilon^{(k)}$ converges.

This proof is consistent with our intuition that all of local models should all converge to a similar point as training continues. 
Notice that if there is no global aggregation throughout the training process ($\alpha=0$), the bound diverges. Since each device only has access to its own data, this agrees with our intuition that the local model parameters would diverge due to data dissimilarity between the devices as the training continues.
\begin{lemma} \label{localDiff}
Under Assumption \ref{assum1}, with $\eta<\frac{2}{\beta}$,
we have, for $t\in\mathcal T_k\setminus\{k\tau\}$,
\begin{align*}
&
\Vert\mathbf w_i(t)-\boldsymbol{c}_k(t)\Vert
\leq
(1+\eta\beta)\Vert\mathbf w_i(t-1)-\boldsymbol{c}_k(t-1)\Vert
+\eta\delta_i.
\end{align*}

\end{lemma}

Since $\boldsymbol{c}_k(t)$ is equivalent to $\mathbf w(t)$ at $t=k\tau-\Delta$ by definition, Lemma \ref{localDiff} quantifies the divergence between the local model parameter $\mathbf w_i(t)$ and the auxiliary centralized GD model parameter $\boldsymbol{c}_k(t)$ at time $t$ in terms of $\epsilon^{(k)}$ from Lemma \ref{epsilon}. It can be observed that the difference between local model $\mathbf w_i((k+1)\tau-\Delta)$ and global model $\mathbf w((k+1)\tau-\Delta)$ induces an exponential growth on the bound, dominated by the effects of $\epsilon^{(k)}$ and the dissimilarity of data distributions among different edge devices $\delta_i$ with respect to $t$. This exponentially growing term disappears when the communication delay is not considered due to the fact that $\mathbf w_i(t)=\boldsymbol{c}_k(t)$ at time $t=k\tau$ for all $k$. In this case, the bound becomes
\begin{align}
    \Vert\mathbf w_i(t)-\boldsymbol{c}_k(t)\Vert\leq& \frac{\delta_i}{\beta}\Big((\beta\eta+1)^{t-k\tau}-1\Big),
\end{align}
which is the same as derived in \cite[Lemma 3]{Wang}.

\begin{lemma} \label{lem5}
Under Assumption \ref{assum1} with $\eta<\frac{2}{\beta}$, we have
\begin{align}
&\Vert\mathbf w(k\tau)-\boldsymbol{c}_k(k\tau)\Vert
\leq
   \nonumber \\&
   \alpha\Delta L\eta+(1-\alpha)
   \Big[((1+\eta\beta)^{\Delta} -1)\epsilon^{(k)}+h(\Delta)\Big].
 \end{align}
\end{lemma}

Lemma \ref{lem5} shows the effect of global aggregation on the gap between $\mathbf w(k\tau)$ and $\boldsymbol{c}_k(k\tau)$. The gap is dominated by two factors $\Delta L\eta$ and $((1+\eta\beta)^{\Delta} -1)\epsilon^{(k)}+h(\Delta)$, which characterize the influence from the global model and the local model, respectively. Notice that when communication delay is not considered (i.e., $\alpha=1$,$\Delta=0$), $\Vert \mathbf w(k\tau)-\boldsymbol{c}_k(k\tau)\Vert\leq 0$ at $t=
k\tau$. This can be expected since global aggregation is performed at time $t=k\tau$ for  all $k$ such that we have $\mathbf w(k\tau-\Delta)$ = $\boldsymbol{c}_k(k\tau-\Delta)$ by definition. 

\begin{proposition} \label{prop1}
Under Assumption \ref{assum1} and $\eta<2/\beta$, we have
\begin{align*}
\numberthis
&\Vert\mathbf w((k+1)\tau-\Delta)-\boldsymbol{c}_k((k+1)\tau-\Delta)\Vert
\\&
   \leq
   \psi(\alpha,k)\triangleq
(1-\alpha)\epsilon^{(k)}\{[1+\eta\beta]^{\tau}-1\}
\\&
+(1-\alpha)h(\tau)+\alpha h(\tau-\Delta)
+\alpha\eta\Delta L[1+\eta\beta]^{\tau-\Delta},
\end{align*}
\end{proposition}

\noindent such that
\begin{align*}
    \Vert F(\mathbf w((k+1)\tau&-\Delta))-F(\boldsymbol{c}_k((k+1)\tau-\Delta))\Vert \leq L\psi(\alpha,k).
\end{align*}

Proposition \ref{prop1} quantifies the upper bound on the divergence between the global model parameter $\mathbf w(t)$ and the auxiliary parameter $\boldsymbol{c}_k(t)$ at $t=(k+1)\tau-\Delta$. 

When the communication delay is negligible ($\Delta=0$) and $\alpha=1$, we see that $\psi(\alpha,k)$ converges to 
\begin{align*}
    \Vert \mathbf w((k+1)\tau)-\boldsymbol{c}_k((k+1)\tau)\Vert\leq \frac{\delta}{\beta}\Big((\beta\eta+1)^{\tau}-1\Big)-\eta\delta\tau,
\end{align*}
which is consistent with the result derived in \cite[Theorem 1]{Wang}, showing that {\tt FedDelAvg} and {\tt FedAvg} are equivalent when $\alpha=1$ with no delay.

When the communication delay is non-negligible $(\Delta>0)$ and $\alpha\leq1$, $\psi(\alpha,k)$ shifts from (i) $\epsilon^{(k)}\{[1+\eta\beta]^{\tau}-1\}+h(\tau)$ to (ii) $\eta\Delta L[1+\eta\beta]^{\tau-\Delta}+h(\tau-\Delta)$ as $\alpha$ goes from $0$ to $1$. Since factors (i) and (ii) are dominated by $\alpha$ and $1-\alpha$ respectively, they correspond to the degree of influence the global and local models have on the bound respectively. We analyze the characteristics of the bound under two cases: $\alpha=1$ and $\alpha<1$. When $\alpha=1$, the bound remains fixed at $t=(k+1)\tau-\Delta$ for all $k$ since the local models are synchronized to the same value after global aggregation. When $\alpha<1$, the bound does not remain the same for all time $t=(k+1)\tau-\Delta$ as $k$ increases since the local models are synchronized to different values after global aggregation. From Proposition \ref{prop1}, it can be observed that the alteration of the bound is characterized by $\epsilon^{(k)}$. Therefore, similar to Lemma \ref{epsilon}, the bound increases in the beginning while the rate of increase continues to decrease as the training process continues until it converges. 

In [15], these results are combined together to prove Theorem \ref{thm1} as follows. We first bound the error between the local and global model by $\epsilon^{(k)}$ at every initialization point of $\boldsymbol{c}_k(t)$ using Lemma \ref{epsilon}. Since $\boldsymbol{c}_k(t)=\mathbf w(t)$ at the initialization points $t=k\tau-\Delta$, we can then bound the divergence between the local model $\mathbf w_i(t)$ and $\boldsymbol{c}_k(t)$ for $t\in\{k\tau-\Delta,...,k\tau-1\}$ by combining $\epsilon^{(k)}$ with the recursive relationship in Lemma \ref{localDiff}. Considering the synchronization phase of {\tt FedDelAvg}, we further bound the divergence between the global model $\mathbf w(t)$ and $\boldsymbol{c}_k(t)$ at $t=k\tau$ using results in Lemma \ref{localDiff}. Finally, applying Lemma \ref{lem5}, we quantify the upper bound on the divergence between the global model $\mathbf w(t)$ and $\boldsymbol{c}_k(t)$ at $t=(k+1)\tau-\Delta$ in Proposition \ref{prop1}. Combining the upper bound obtained in Proposition \ref{prop1} with the convergence property of $\boldsymbol{c}_k(t)$ then completes the proof of Theorem \ref{thm1}.


\section{Experimental Evaluation}
To verify our theoretical results, we conduct numerical experiments to examine the effect of communication delay on the convergence of federated learning. The simulation is carried out using the TensorFlow Federated (TFF) framework \cite{tf}. Considering the fact that cloud would only have access to the global model at $t=k\tau-\Delta$ for all $k$ when the delay between edge and cloud are carefully considered, we evaluate the averaged model $\Delta$ iterations before each global synchronization on the corresponding global loss function. 

We consider a federated learning system with $N=10$ edge devices. The number of local update steps between two global aggregations is set to $\tau=10$, communication delay is set to $\Delta=9$, and the total number of global aggregation steps is set to $K=100$.\\
\textbf{Dataset.} The MNIST dataset \cite{LeCun} containing 70K images (60K for training and 10K for testing) of hand-written digits is considered in the simulation. We distribute the dataset among the edge devices in a manner such that each obtains a subset corresponding to a specific writer. Since each writer has an unique writing style, the data exhibits dissimilarity ($\delta>0$) among devices (data dissimilarity is referred to as 
“non-iid" in the literature, see \cite{Kairouz}).\\
\textbf{ML model.} We consider a generic multinominal logistic regression machine learning model to predict the label of each image out of $s=10$ possible classes. The cross-entropy loss $f_i(\mathbf x,y;\mathbf w)= -\sum\limits_{j=1}^{s}\{y=j\}\log e^{\mathbf w_j^T \mathbf x}\Big(\sum\limits_{l=1}^{s}e^{\mathbf w_l^T \mathbf x}\Big)^{-1}$, which satisfies Assumptions \ref{assum1} and \ref{def:2}, is applied as the loss function at each edge device. During the local updating process, each edge device performs gradient descent with full batch size and a fixed learning rate $\eta=0.02$.

We study the effects of the two key parameters -- the weighting $\alpha$ which we control, and the communication delay $\Delta$ which is an artifact of the system --  on Federated Delayed Average Learning. The convergence of the testing accuracy, defined as the fraction of classes predicted correctly relative to the total number of predictions, is demonstrated with respect to different values of $\alpha$ and $\Delta$.

\begin{figure}
\centering
 \includegraphics[scale=0.23]{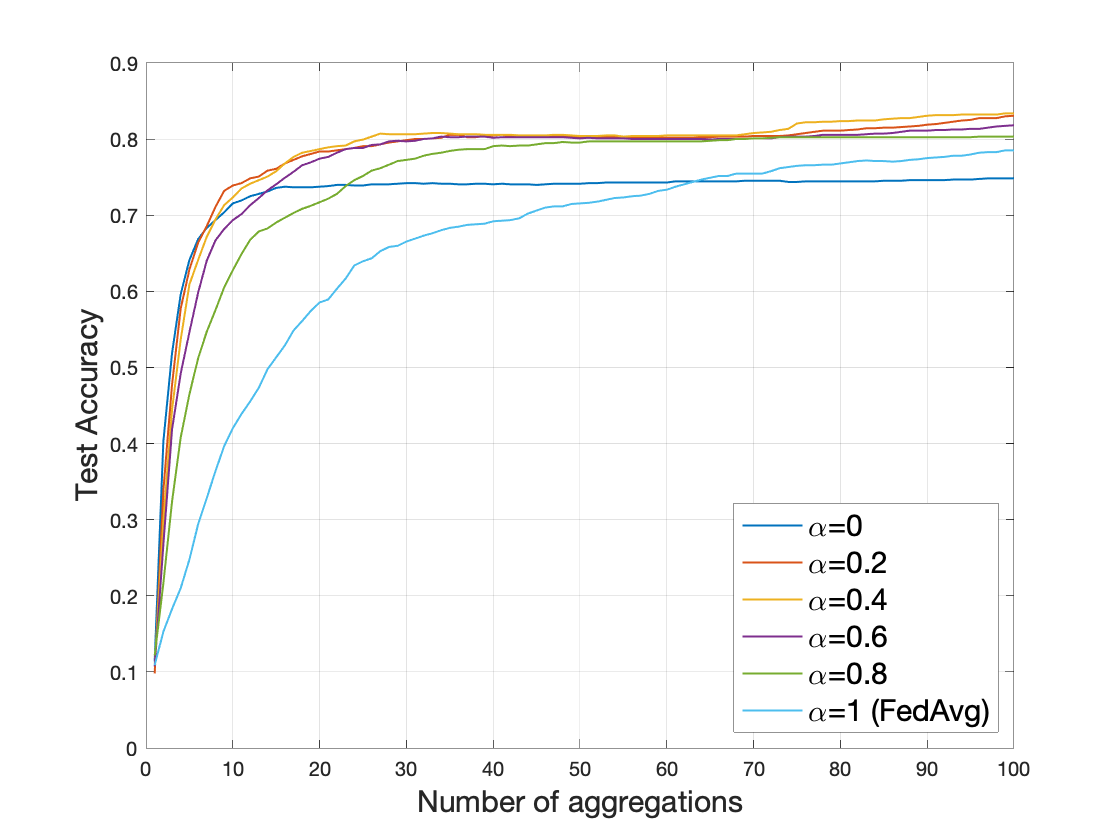}
     \caption{Accuracy w.r.t different $\alpha$, for a fixed $\Delta = 9$. Selection of $\alpha$ strongly affects rate at which accuracy converges.}
     \label{fig3}
\end{figure}

Fig.2 depicts the improvement in testing accuracy by aggregation for different values of $\alpha$. The convergence speed varies with respect to different values of $\alpha$. Compared with {\tt FedAvg} ($\alpha=1$), {\tt FedDelAvg} gains the best improvement with $\alpha=0.2$: the model reaches an accuracy of roughly $80\%$ in $78\%$ fewer training iterations.

Observing the best selection of $\alpha$ in Fig.2, we set $\alpha=0.2$ and compare {\tt FedDelAvg} with {\tt FedAvg} when delay is either negligible or non-negligible ($\Delta=0$ and $\Delta=9$) in terms of accuracy.
As shown in Fig. \ref{fig4}, when communication delay between edge and cloud is negligible ($\Delta=0$), {\tt FedAvg} obtains the best performance in terms of convergence rate, verifying our conclusion in Theorem \ref{thm1} that $\alpha=1$ is the  optimum when $\Delta=0$. To further demonstrate the delay robustness of {\tt FedDelAvg}, we set the performance of {\tt FedAvg} with no delay ($\Delta=0$) as benchmark and compare it with {\tt FedDelAvg} under delay ($\Delta=9$) when $\alpha$ is optimized. 
We observe that, even when the delay is non-negligible, {\tt FedDelAvg} achieves an accuracy of 80\% while only requiring 10\% extra training iterations compared with the benchmark. After 100 aggregations, {\tt FedDelAvg} achieves an accuracy within 3\% of the benchmark, whereas {\tt FedAvg} has a severely degraded performance, thus demonstrating the delay-robustness of the proposed algorithm.

\afterpage{%
\begin{figure}
\centering
    \includegraphics[scale=0.23]{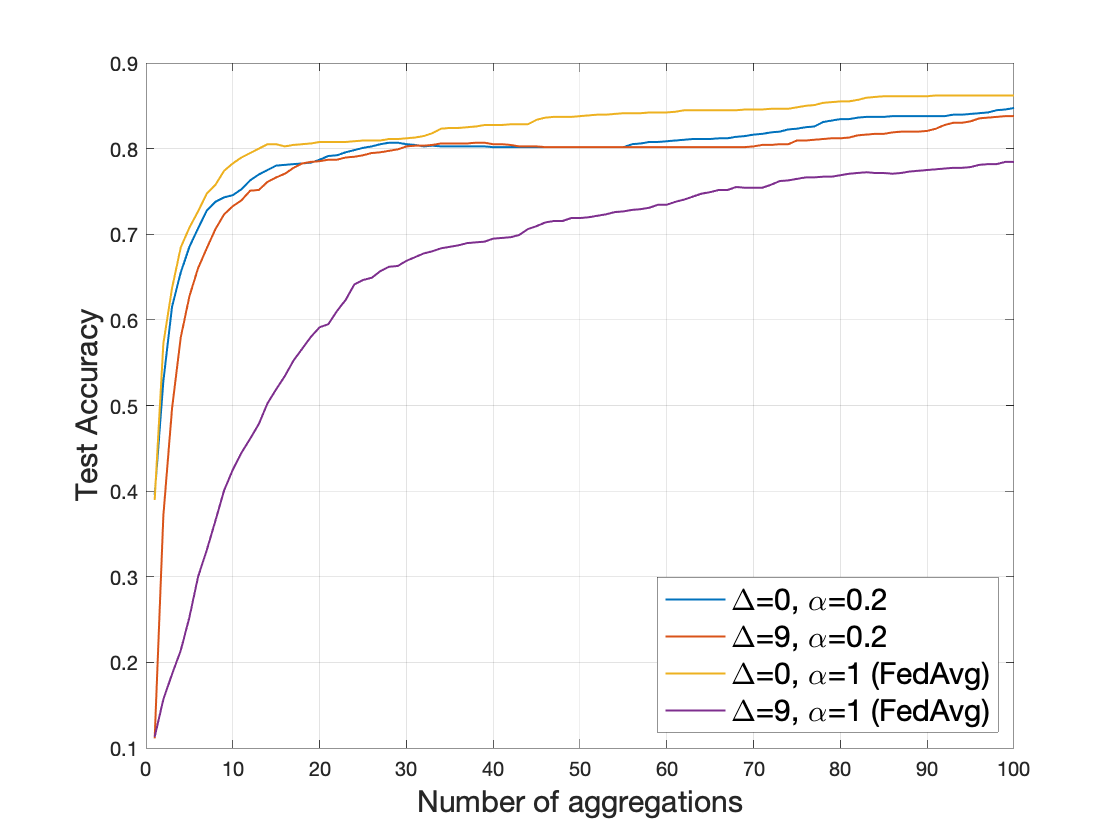}
     \caption{Accuracy w.r.t different communication delays. {\tt FedDelAvg} under delay is able to obtain a similar convergence rate to the case of no delay when $\alpha$ is tuned.}
     \label{fig4}
\end{figure}
}

\section{Conclusion}
This paper proposed {\tt FedDelAvg}, a generalized federated averaging algorithm that incorporates communication delay in edge networks. Analysis of the convergence bound of {\tt FedDelAvg} was conducted with respect to dissimilarities of local data at each device.
The experimental results demonstrated the impact of delays on federated learning and the delay-robustness of {\tt FedDelAvg}. Overall, we found that the global model converges significantly faster when the synchronization weighting is optimized for the delay compared with existing FL algorithms where local models are not considered during synchronization.



\bibliographystyle{IEEEtran}
\bibliography{ref}

\appendices
\setcounter{lemma}{0}
\setcounter{proposition}{0}
\setcounter{theorem}{0}
\section{}

\subsection{Proof of Lemma \ref{norm}}
\begin{lemma}
\label{norm}
Under any $F$ satisfying Assumption \ref{assum1}, we have
$$
\Vert
[\mathbf w_1-\eta\nabla F(\mathbf w_1)]
-[\mathbf w_2- \eta\nabla F(\mathbf w_2)]
\Vert
\leq \sqrt{1+\beta^2\eta^2}\Vert\mathbf w_1-\mathbf w_2\Vert
$$$$
\leq (1+\beta\eta)\Vert\mathbf w_1-\mathbf w_2\Vert.
$$
\end{lemma}

\begin{proof}
Note that, from the convexity of $F$, it follows that
\begin{align}
&F(\mathbf w_2)\geq F(\mathbf w_1)+(\mathbf w_2-\mathbf w_1)^T\nabla F(\mathbf w_1)\\
&
F(\mathbf w_1)\geq F(\mathbf w_2)+(\mathbf w_1-\mathbf w_2)^T\nabla F(\mathbf w_2).
\end{align}
Summing both inequalities, we obtain
$$
(\mathbf w_2-\mathbf w_1)^T(\nabla F(\mathbf w_2)-\nabla F(\mathbf w_1)\geq0.
$$
Then,
\begin{align}
&\Vert
[\mathbf w_1-\eta\nabla F(\mathbf w_1)]
-[\mathbf w_2- \eta\nabla F(\mathbf w_2)]
\Vert^2
\\&
=
\Vert
\mathbf w_1-\mathbf w_2
\Vert^2
+\eta^2\Vert
\nabla F(\mathbf w_1)-\nabla F(\mathbf w_2)
\Vert^2
\\&
-2[\mathbf w_2-\mathbf w_1]
[\nabla F(\mathbf w_2)-\eta\nabla F(\mathbf w_1)]
\\&
\leq
(1+\beta^2\eta^2)\Vert
\mathbf w_1-\mathbf w_2
\Vert^2,
\end{align}
where the last inequality follows from the fact that 
$F$ is $\beta$-smooth.
The result of the lemma thus follows.
\end{proof}

\subsection{Proof of Lemma \ref{gradNorm}}
\begin{lemma} \label{gradNorm}
Under Assumption \ref{assum1},
\begin{align}
\Vert\nabla F_i(\mathbf w)\Vert &\leq L,\ \forall i,\forall\mathbf w.
\end{align}
\end{lemma}

\begin{proof}  
Note that convex and $L$-Lipschitz conditions imply, $\forall \mathbf w',\mathbf w$,
\begin{align} \label{eq:15}
\langle\mathbf w'-\mathbf w,\nabla F_i(\mathbf w)\rangle\leq &F_i(\mathbf w')-F_i(\mathbf w)\\
F_i(\mathbf w')-F_i(\mathbf w)\leq &L\Vert\mathbf w'-\mathbf w\Vert
\end{align}
Let $\mathbf w'=\mathbf w-\nabla F_i(\mathbf w)$, then we prove that
\begin{align}
\Vert\nabla F(\mathbf w)\Vert\leq L.
\end{align}
\end{proof}

\subsection{Proof of Lemma \ref{epsilon}}
\begin{lemma} \label{epsilon}
Under Assumption \ref{assum1}, with learning rate $\eta < \frac{2}{\beta}$, we have
\begin{align} \label{eq:34}
    \Vert\mathbf w_i(k\tau-\Delta)-\mathbf w(k\tau-\Delta)\Vert \leq \epsilon_i^{(k)}
\end{align}

for all $k \in\{0,1,\dots,K\}$, where
\begin{align}
\epsilon^{(k)}=
[1-(1-\alpha)^{k}]2\eta L(\tau/\alpha-\Delta).
\end{align}

\end{lemma}

\begin{proof} 
From \eqref{8}, we find that
$k\tau<t<(k+1)\tau$
\begin{align}
\mathbf w_i((k+1)\tau-\Delta)
 =& 
\mathbf w_i(k\tau) 
-\eta\sum_{r=1}^{\tau-\Delta}
\nabla F_i(\mathbf w_i((k+1)\tau-\Delta-r))
;
\end{align}
moreover, from \eqref{9}
    \begin{align}
    \label{eq:29}
    \nonumber
&        \mathbf w_i(k\tau)
        = \alpha \mathbf w(k\tau-\Delta)+(1-\alpha)
        [\mathbf w_i(k\tau-1)-\eta\nabla F_i(\mathbf w_i(k\tau-1))]
        \\&\nonumber
        =\alpha \mathbf w(k\tau-\Delta)+(1-\alpha)\mathbf w_i(k\tau-\Delta)
        \\
        &- (1-\alpha)\eta\sum\limits_{r=1}^{\Delta}\nabla F_i(\mathbf w_i(k\tau-r));        
    \end{align}
after    combining, we obtain
\begin{align}
\nonumber
&\mathbf w_i((k+1)\tau-\Delta)
 =
\alpha \mathbf w(k\tau-\Delta)
+(1-\alpha)\mathbf w_i(k\tau-\Delta)
\\&\nonumber
- (1-\alpha)\eta\sum\limits_{r=1}^{\Delta}
\nabla F_i(\mathbf w_i(k\tau-r))
\\&     
-\eta\sum_{r=1}^{\tau-\Delta}
\nabla F_i(\mathbf w_i((k+1)\tau-\Delta-r)).
\end{align}
Therefore,
\begin{align*}
&\mathbf w_i((k+1)\tau-\Delta)
-\mathbf w((k+1)\tau-\Delta)
\\&
=
\mathbf w_i((k+1)\tau-\Delta)
-\sum_j\rho_j\mathbf w_j((k+1)\tau-\Delta)
\\&
=
(1-\alpha)[\mathbf w_i(k\tau-\Delta)-\mathbf w(k\tau-\Delta)]
\\&
- (1-\alpha)\eta(1-\rho_i)\sum\limits_{r=1}^{\Delta}
\nabla F_i(\mathbf w_i(k\tau-r))
\\&
+(1-\alpha)\eta\sum\limits_{r=1}^{\Delta}
\sum_{j\neq i}\rho_j\nabla F_j(\mathbf w_j(k\tau-r))
\\&     
- \eta\sum_{r=1}^{\tau-\Delta}
(1-\rho_i)\nabla F_i(\mathbf w_i((k+1)\tau-\Delta-r))
\\&\qquad
+\eta\sum_{r=1}^{\tau-\Delta}
\sum_{j\neq i}\rho_j\nabla F_i(\mathbf w_j((k+1)\tau-\Delta-r)).
\numberthis
\end{align*}
Computing the norm and using the triangular inequality, we then obtain
\begin{align*}
&\Vert\mathbf w_i((k+1)\tau-\Delta)
-\mathbf w((k+1)\tau-\Delta)\Vert
\\&
\leq
(1-\alpha)\Vert
\mathbf w_i(k\tau-\Delta)-\mathbf w(k\tau-\Delta)\Vert
\\&
+(1-\alpha)\eta(1-\rho_i)\sum\limits_{r=1}^{\Delta}
\Vert\nabla F_i(\mathbf w_i(k\tau-r))\Vert
\\&
+(1-\alpha)\eta\sum\limits_{r=1}^{\Delta}
\sum_{j\neq i}\rho_j\Vert\nabla F_j(\mathbf w_j(k\tau-r))\Vert
\\&     
+\eta\sum_{r=1}^{\tau-\Delta}
(1-\rho_i)\Vert\nabla F_i(\mathbf w_i((k+1)\tau-\Delta-r))\Vert
\\&\qquad
+\eta\sum_{r=1}^{\tau-\Delta}
\sum_{j\neq i}\rho_j
\Vert\nabla F_i(\mathbf w_j((k+1)\tau-\Delta-r))\Vert.
\numberthis
\end{align*}
Finally, using Lemma \ref{gradNorm} we obtain
\begin{align*}
&\Vert\mathbf w_i((k+1)\tau-\Delta)
-\mathbf w((k+1)\tau-\Delta)\Vert
\\&
\leq
(1-\alpha)\Vert
\mathbf w_i(k\tau-\Delta)-\mathbf w(k\tau-\Delta)\Vert
\\&
+2\eta L(1-\rho_i)(\tau-\alpha\Delta).
  \numberthis
\end{align*}
By induction, we then find
\begin{align*}
&
\Vert\mathbf w_i(k\tau-\Delta)-\mathbf w(k\tau-\Delta)\Vert
\leq
(1-\alpha)^{k}
\Vert
\mathbf w_i(-\Delta)-\mathbf w(-\Delta)\Vert
\\&
+[1-(1-\alpha)^{k}]
2\eta L(\tau/\alpha-\Delta)
\triangleq \epsilon^{(k)},
  \numberthis
\end{align*}
and we obtain the desired result by initializing 
$\mathbf w_i(-\Delta)=\mathbf w(-\Delta),\forall i$.

\end{proof}

\subsection{Proof of Lemma \ref{localDiff}}
\begin{lemma} \label{localDiff}
Under Assumption \ref{assum1}, with learning rate $\eta<\frac{2}{\beta}$,
we have, for $t\in(k\tau-\Delta,(k+1)\tau-\Delta),t\neq k\tau$,
\begin{align*}
&
\Vert\mathbf w_i(t)-\boldsymbol{c}_k(t)\Vert
\leq
(1+\eta\beta)\Vert\mathbf w_i(t-1)-\boldsymbol{c}_k(t-1)\Vert
+\eta\delta_i.
\end{align*}

\end{lemma}

\begin{proof} 
Let $t\in(k\tau-\Delta,(k+1)\tau-\Delta),t\neq k\tau$. We have
\begin{align*}
&\mathbf w_i(t)-\boldsymbol{c}_k(t)
=
\mathbf w_i(t-1)-\boldsymbol{c}_k(t-1)
\\&
+\eta[\nabla F_i(\boldsymbol{c}_k(t-1))
-\nabla F_i(\mathbf w_i(t-1))]
\\&
+\eta[\nabla F(\boldsymbol{c}_k(t-1))
-\nabla F_i(\boldsymbol{c}_k(t-1))].
\end{align*}
Taking the norm and using the triangular inequality, we then obtain
\begin{align*}
&\Vert\mathbf w_i(t)-\boldsymbol{c}_k(t)\Vert
\leq
\Vert\mathbf w_i(t-1)-\boldsymbol{c}_k(t-1)\Vert
\\&
+\eta\Vert\nabla F_i(\boldsymbol{c}_k(t-1))
-\nabla F_i(\mathbf w_i(t-1))\Vert
\\&
+\eta\Vert\nabla F(\boldsymbol{c}_k(t-1))
-\nabla F_i(\boldsymbol{c}_k(t-1))\Vert.
\end{align*}
The result of the lemma is then found by using the $\beta$-smoothness of 
$F_i(\cdot)$ and Assumption \ref{def:2}.
\end{proof}

\subsection{Proof of Lemma \ref{lem5}}
\begin{lemma} \label{lem5}
Under Assumption \ref{assum1} with learning rate $\eta<2/\beta$, we have
\begin{align}
&\Vert\mathbf w(k\tau)-\boldsymbol{c}_k(k\tau)\Vert
\leq
\eta\Delta\left[\alpha L+(1-\alpha)\beta\epsilon^{(k)}\right]
\\&
  +(1-\alpha)\left(\frac{\beta}{\delta}\epsilon^{(k)}+1\right)h(\Delta).
  \\&
  =\alpha \Delta L\eta+(1-\alpha)
  \Big[((1+\eta\beta)^{\Delta} -1)\epsilon^{(k)}+h(\Delta)\Big].
     \end{align}
\end{lemma}

\begin{proof} 

\noindent 
Note that, using \eqref{eq:29} 
 \begin{align}
&\mathbf w(k\tau)=\sum_i\rho_i\mathbf w_i(k\tau)
        \\&
        = \mathbf w(k\tau-\Delta)- (1-\alpha)\eta\sum\limits_{r=1}^{\Delta}
        \sum_i\rho_i
        \nabla F_i(\mathbf w_i(k\tau-r)).      
    \end{align}
    Moreover,
    \begin{align}
        \boldsymbol{c}_k(k\tau)&= \boldsymbol{c}_k(k\tau-\Delta)-\eta\sum\limits_{r=1}^{\Delta}\sum\limits_i \rho_i \nabla F_i(\boldsymbol{c}_k(k\tau-r)).
    \end{align}
    Therefore, we obtain
    \begin{align}
&\mathbf w(k\tau)-\boldsymbol{c}_k(k\tau)
= 
\eta\alpha
\sum\limits_{r=1}^{\Delta}\sum_i\rho_i\nabla F_i(\boldsymbol{c}_k(k\tau-r))
        \\&
        - (1-\alpha)\eta\sum\limits_{r=1}^{\Delta}
        \sum_i\rho_i[\nabla F_i(\mathbf w_i(k\tau-r))
        -\nabla F_i(\boldsymbol{c}_k(k\tau-r))]
    \end{align}
    where we used the fact that $\boldsymbol{c}_k(k\tau-\Delta)=\mathbf w(k\tau-\Delta)$.
    Taking the norm and using the triangular inequality, we then obtain
    \begin{align}
&\Vert\mathbf w(k\tau)-\boldsymbol{c}_k(k\tau)\Vert
\leq
\eta\alpha
\sum\limits_{r=1}^{\Delta}\sum_i\rho_i\Vert
\nabla F_i(\boldsymbol{c}_k(k\tau-r))\Vert
        \\&
        +(1-\alpha)\eta\beta\sum\limits_{r=1}^{\Delta}
        \sum_i\rho_i\Vert\mathbf w_i(k\tau-r)
        -\boldsymbol{c}_k(k\tau-r)\Vert.
    \end{align}
    where we used the $\beta$-smoothness of $F_i$ to further upper bound
    $\Vert\nabla F_i(\mathbf w_i(k\tau-r))
        -\nabla F_i(\boldsymbol{c}_k(k\tau-r))\Vert$.
        Using Lemma \ref{gradNorm} and Lemma \ref{localDiff}, we can further bound
        \begin{align}
&\Vert\mathbf w(k\tau)-\boldsymbol{c}_k(k\tau)\Vert
\leq
\eta\alpha L\Delta
 +(1-\alpha)\eta\beta \epsilon^{(k)}\sum_{r=1}^{\Delta}(1+\eta\beta)^{\Delta-r}
        \\&
+\frac{\delta}{\beta}
(1-\alpha)\eta\beta\sum_{r=1}^{\Delta}[(1+\eta\beta)^{\Delta-r}-1],
     \end{align}
     yielding the result in the lemma after algebraic steps.
\end{proof}

\subsection{Proof of Proposition \ref{prop1}}
\begin{proposition} \label{prop1}
Under Assumption \ref{assum1} and learning rate $\eta<2/\beta$, we have
\begin{align*}
\numberthis
&\Vert\mathbf w((k+1)\tau-\Delta)-\boldsymbol{c}_k((k+1)\tau-\Delta)\Vert
\\&
  \leq
  \psi(\alpha,k)\triangleq
(1-\alpha)\epsilon^{(k)}\{[1+\eta\beta]^{\tau}-1\}
\\&
+(1-\alpha)h(\tau)+\alpha h(\tau-\Delta)
\\&
+\alpha\eta\Delta L[1+\eta\beta]^{\tau-\Delta},
\end{align*}
\end{proposition}

\begin{proof} 
Let $t\in(k\tau-\Delta,(k+1)\tau-\Delta]$. Then from \eqref{9} we have
\begin{align}
  \mathbf w_i =& 
            \alpha_t \mathbf w(k\tau-\Delta)
            \\&
+(1-\alpha_t)
\left[\mathbf w_i(t-1) - \eta\nabla F_i(\mathbf w_i(t-1))\right]
\end{align}
and
\begin{align}
\boldsymbol{c}_k(t) =& 
\boldsymbol{c}_k(t-1)
- \eta\nabla F(\boldsymbol{c}_k(t-1)).
\end{align}
Using the fact that 
$$
\boldsymbol{c}_k(k\tau-1)=
\mathbf w(k\tau-\Delta)
- \eta\sum_{r=0}^{\Delta-2}\nabla F(\boldsymbol{c}_k(k\tau-\Delta+r)),
$$
it follows that 
\begin{align*}
&   \mathbf w(t)-\boldsymbol{c}_k(t)
  =
(1-\alpha_t)[\mathbf w(t-1)-\boldsymbol{c}_k(t-1)]
\\&
- (1-\alpha_t)\eta\sum_i\rho_i
[\nabla F_i(\mathbf w_i(t-1))-\nabla F_i(\boldsymbol{c}_k(t-1))]
\\&
+\eta\alpha_t\sum_{r=0}^{\Delta-1}\nabla F(\boldsymbol{c}_k(k\tau-\Delta+r))
\end{align*}
Taking the norm, using the triangular inequality,
  $\beta$-smoothness of $F_i$,
 and Lemma \ref{gradNorm} to bound $\Vert\nabla F(\boldsymbol{c}_k(t))\Vert$,
  we obtain the inequality
\begin{align*}
&   \Vert\mathbf w(t)-\boldsymbol{c}_k(t)\Vert
\\&
  \leq
(1-\alpha_t)\Vert\mathbf w(t-1)-\boldsymbol{c}_k(t-1)\Vert
\\&
+(1-\alpha_t)\eta\beta\sum_i\rho_i
\Vert\mathbf w_i(t-1)-\boldsymbol{c}_k(t-1)\Vert
\\&
+\alpha_t\eta L\Delta.
\end{align*}
By induction, we then obtain, for $t\in[k\tau-\Delta,k\tau)$
($\alpha_t=0$ for all such $t$)
\begin{align*}
& \Vert\mathbf w(t)-\boldsymbol{c}_k(t)\Vert
  \leq
\eta\beta\sum_{\ell=k\tau-\Delta}^{t-1}
\sum_i\rho_i
\Vert\mathbf w_i(\ell)-\boldsymbol{c}_k(\ell)\Vert,
\end{align*}
where we used the fact that
$\boldsymbol{c}_k(k\tau-\Delta)=\mathbf w(k\tau-\Delta)$,
and
for $t\in[k\tau,(k+1)\tau-\Delta]$ (note that $\alpha_{k\tau}=\alpha$ and
$\alpha_t=0,\forall t>k\tau$)
\begin{align*}
&\Vert\mathbf w(t)-\boldsymbol{c}_k(t)\Vert
\\&
  \leq
(1-\alpha)\eta\beta\sum_{\ell=k\tau-\Delta}^{k\tau-1}
\sum_i\rho_i\Vert\mathbf w_i(\ell)-\boldsymbol{c}_k(\ell)\Vert,
\\&
+\eta\beta \sum_{\ell=k\tau}^{t-1}
\sum_i\rho_i\Vert\mathbf w_i(\ell)-\boldsymbol{c}_k(\ell)\Vert
+\alpha\eta L\Delta.
\end{align*}
Therefore,
\begin{align*}
\label{eqdfgh}
&\Vert\mathbf w((k+1)\tau-\Delta)-\boldsymbol{c}_k((k+1)\tau-\Delta)\Vert
\\&
  \leq
(1-\alpha)\eta\beta\sum_{\ell=k\tau-\Delta}^{k\tau-1}
\sum_i\rho_i\Vert\mathbf w_i(\ell)-\boldsymbol{c}_k(\ell)\Vert,
\\&
+\eta\beta \sum_{\ell=k\tau}^{(k+1)\tau-\Delta-1}
\sum_i\rho_i\Vert\mathbf w_i(\ell)-\boldsymbol{c}_k(\ell)\Vert
+\alpha\eta L\Delta.
\end{align*}
We now bound the term $\sum_i\rho_i
\Vert\mathbf w_i(\ell)-\boldsymbol{c}_k(\ell)\Vert$. Note that 
$\sum_i\rho_i
\Vert\mathbf w_i(k\tau-\Delta)-\boldsymbol{c}_k(k\tau-\Delta)\Vert
=\sum_i\rho_i
\Vert\mathbf w_i(k\tau-\Delta)-\mathbf w(k\tau-\Delta)\Vert
\leq\epsilon^{(k)}$ (Lemma \ref{epsilon}). For $\ell\in(k\tau-\Delta,(k+1)\tau-\Delta]$ we then have
\begin{align*}
&\mathbf w_i(\ell)-\boldsymbol{c}_k(\ell)
=
(1-\alpha_\ell)
\left[\mathbf w_i(\ell-1)-\boldsymbol{c}_k(\ell-1)\right]
\\&
 - (1-\alpha_\ell)\eta
\left[\nabla F_i(\mathbf w_i(\ell-1))
-\nabla F_i(\boldsymbol{c}_k(\ell-1))
\right]
\\&
 - (1-\alpha_\ell)\eta
\left[\nabla F_i(\boldsymbol{c}_k(\ell-1))-\nabla F(\boldsymbol{c}_k(\ell-1))\right]
\\&
+\alpha_\ell\eta\sum_{r=0}^{\Delta-1}\nabla F(\boldsymbol{c}_k(k\tau-\Delta+r)).
\end{align*}
Taking the norm, using the triangular inequality,
 Lemma \ref{norm}, $\beta$-smoothness of $F_i$, 
 definition \ref{def:2}
 and computing the sum $\sum_i\rho_i$,
  we obtain the inequality
\begin{align}
&\sum_i\rho_i\Vert\mathbf w_i(\ell)-\boldsymbol{c}_k(\ell)\Vert
\\&
\leq
(1-\alpha_\ell)[1+\eta\beta]
\sum_i\rho_i\Vert\mathbf w_i(\ell-1)-\boldsymbol{c}_k(\ell-1)\Vert
\\&
+(1-\alpha_\ell)\eta\delta
+\alpha_\ell\eta\Delta L
\end{align}
Using induction,
 it then follows, for $\ell\in[k\tau-\Delta,k\tau-1]$ ($\alpha_\ell=0$)
\begin{align}
&
\sum_i\rho_i\Vert\mathbf w_i(\ell)-\boldsymbol{c}_k(\ell)\Vert
\\&
\leq
[1+\eta\beta]^{\ell-k\tau+\Delta}\epsilon^{(k)}
+\delta\frac{[1+\eta\beta]^{\ell-k\tau+\Delta}-1}{\beta},
\end{align}
and for $\ell\in[k\tau,(k+1)\tau-\Delta]$,
\begin{align}
&
\sum_i\rho_i\Vert\mathbf w_i(\ell)-\boldsymbol{c}_k(\ell)\Vert
\\&
\leq
(1-\alpha)[1+\eta\beta]^{\ell-k\tau+\Delta}\epsilon^{(k)}
\\&
+(1-\alpha)\delta[1+\eta\beta]^{\ell-k\tau}\frac{[1+\eta\beta]^{\Delta}-1}{\beta}
\\&
+\delta\frac{[1+\eta\beta]^{\ell-k\tau}-1}{\beta}
+\alpha\eta\Delta L[1+\eta\beta]^{\ell-k\tau}.
\end{align}
It then follows that
\begin{align*}
&
\sum_{\ell=k\tau-\Delta}^{k\tau-1}
\sum_i\rho_i\Vert\mathbf w_i(\ell)-\boldsymbol{c}_k(\ell)\Vert
\\&
\leq
\epsilon^{(k)}\frac{[1+\eta\beta]^{\Delta}-1}{\eta\beta}
+\frac{h(\Delta)}{\eta\beta},
\end{align*}
and
\begin{align}
&\sum_{\ell=k\tau}^{(k+1)\tau-\Delta-1}
\sum_i\rho_i\Vert\mathbf w_i(\ell)-\boldsymbol{c}_k(\ell)\Vert
\\&
\leq
(1-\alpha)[1+\eta\beta]^{\Delta}\epsilon^{(k)}
\frac{[1+\eta\beta]^{\tau-\Delta}-1}{\eta\beta}
\\&
+(1-\alpha)\frac{h(\tau)-h(\Delta)}{\eta\beta}
+\alpha\frac{h(\tau-\Delta)}{\eta\beta}
\\&
+\alpha\Delta L\frac{[1+\eta\beta]^{\tau-\Delta}-1}{\beta}
\end{align}
Combining these bounds with \eqref{eqdfgh}, we finally obtain
\begin{align*}
\numberthis
&\Vert\mathbf w((k+1)\tau-\Delta)-\boldsymbol{c}_k((k+1)\tau-\Delta)\Vert
\\&
  \leq
(1-\alpha)\epsilon^{(k)}\{[1+\eta\beta]^{\tau}-1\}
\\&
+(1-\alpha)h(\tau)+\alpha h(\tau-\Delta)
\\&
+\alpha\eta\Delta L[1+\eta\beta]^{\tau-\Delta},
\end{align*}
thus proving the Lemma.
\end{proof}

\subsection{Proof of Proposition \ref{prop2}}
\begin{proposition} \label{prop2}
Let 
$$
\omega = \frac{1}{\max_{k\in\{0,\dots,K-1\}} \Vert \boldsymbol{c}_k(k\tau-\Delta)-\boldsymbol{w^*}\Vert^2}
$$
Under Assumption \ref{assum1}, and if the following conditions are satisfied,
\begin{enumerate}
    \item $\eta < \frac{2}{\beta}$
    \item $T\eta\phi-\frac{L\Psi(\alpha)}{\xi^2}>0$
    \item $F(\boldsymbol{c}_k((k+1)\tau-\Delta))-F(\boldsymbol{w^*}) \geq \xi\ $ \textit{for all} $k$
    \item $F(\mathbf w((K+1)\tau-\Delta))-F(\boldsymbol{w^*}) \geq \xi$
\end{enumerate}

\noindent for some $\xi >0$, the convergence upper bound of {\tt FedDelAvg} is
\begin{align} 
    F(\mathbf w((K+1)\tau-\Delta))-F(\boldsymbol{w^*})\leq&\frac{1}{T\eta\phi-\frac{L\Psi(\alpha)}{\xi^2}}.\\
\end{align}
\noindent where $\Psi(\alpha)=\sum\limits_{k=1}^{K}\psi(\alpha,k)$.
\end{proposition}

\begin{proof} 
First, note that, if $\omega=\infty$, i.e.,
 $\boldsymbol{c}_k(k\tau-\Delta)=\boldsymbol{w^*},\forall k$, then
$\mathbf w((K+1)\tau-\Delta))=\boldsymbol{c}_{[K+1]}((K+1)\tau-\Delta)
=\boldsymbol{w^*}$,
 hence $F(\mathbf w((K+1)\tau-\Delta))=F(\boldsymbol{w^*})$.
 Now, let us consider the case $\omega<\infty$.
  For every interval $k$ and $t\in [k\tau-\Delta,(k+1)\tau-\Delta]$, we define
the sub-optimality gap of the centralized GD scheme,
    \begin{align}
        \Gamma_{[k]}(t)=F(\boldsymbol{c}_k(t))-F(\boldsymbol{w^*}).
    \end{align}
    Note that $\Gamma_{[k]}(t)\geq 0,\forall k$.  
    Since $\mathbf w((K+1)\tau-\Delta))=\boldsymbol{c}_{[K+1]}((K+1)\tau-\Delta)$, we want to prove that
    \begin{align}
    \label{goal}
\Gamma_{[K+1]}((K+1)\tau-\Delta))^{-1}\geq T\eta\phi-\frac{L\Psi(\alpha)}{\xi^2},
    \end{align}
(trivially satisfied if $\Gamma_{[K+1]}((K+1)\tau-\Delta))=0$).
To determine this bound, note that
\cite[Lemma 6]{Wang}
    \begin{align}
    \label{53}
&\Gamma_{[k]}(t+1)^{-1}-\Gamma_{[k]}(t)^{-1}\geq
                \frac{\eta\left(1-\frac{\beta\eta}{2}\right)}{\Vert\boldsymbol{c}_k(t)-\boldsymbol{w^*}\Vert^2}
                \\&
                \geq
                \frac{\eta\left(1-\frac{\beta\eta}{2}\right)}{\max_k\Vert\boldsymbol{c}_k(t)-\boldsymbol{w^*}\Vert^2}
                =
 \eta\omega\Big(1-\frac{\beta\eta}{2}\Big)
 =\eta\phi,
    \end{align}
    and therefore
    \begin{align}
    \label{cgfh}
&
\Gamma_{[k]}((k+1)\tau-\Delta)^{-1}-\Gamma_{[k]}(k\tau-\Delta)^{-1}
\\&
=
\sum_{t=k\tau-\Delta}^{(k+1)\tau-\Delta-1}
\left[\Gamma_{[k]}(t+1)^{-1}-\Gamma_{[k]}(t)^{-1}\right]
\\&
\geq \tau\eta\phi.
    \end{align}
    It follows that
    \begin{align}
&
\sum_{k=1}^{K}
\left[\Gamma_{[k]}((k+1)\tau-\Delta)^{-1}-\Gamma_{[k]}(k\tau-\Delta)^{-1}\right]
\\&
=
\Gamma_{[K+1]}((K+1)\tau-\Delta))^{-1}-\Gamma_{[1]}(\tau-\Delta)^{-1}
\\&
-\sum_{k=1}^{K}\left[
\Gamma_{[k+1]}((k+1)\tau-\Delta)^{-1}
-\Gamma_{[k]}((k+1)\tau-\Delta)^{-1}
\right]
\\&
\geq T\eta\phi,
    \end{align}
    where $K\tau=T$.
    Therefore, to prove \eqref{goal}, it is sufficient to show that
    \begin{align}
    \label{goal2}    
\sum_{k=1}^{K}\left[
\Gamma_{[k]}((k+1)\tau-\Delta)^{-1}-\Gamma_{[k+1]}((k+1)\tau-\Delta)^{-1}
\right]
    \leq \frac{L\Psi(\alpha)}{\xi^2},
    \end{align}
    which we now prove. Note that, since $\Psi(\alpha)=\sum\limits_{k=1}^{K}\psi(\alpha,k)$, a sufficient condition which implies \eqref{goal2} is
    \begin{align}
    \label{goal4}
&\Gamma_{[k+1]}((k+1)\tau-\Delta)-\Gamma_{[k]}((k+1)\tau-\Delta)
\\&
    \leq \frac{L\psi(\alpha,k)}{\xi^2}\Gamma_{[k]}((k+1)\tau-\Delta)\Gamma_{[k+1]}((k+1)\tau-\Delta).
    \end{align}    
    Note that, from conditions (3) and (4) of the proposition statement,
    \begin{align}
    \label{c34}
    &\Gamma_{[k]}((k+1)\tau-\Delta))\geq \xi,\ \forall k,\\
    &\Gamma_{[K+1]}((K+1)\tau-\Delta)\geq \xi.
    \end{align}  
Moreover, from \eqref{cgfh} with $k<K-1$,
\begin{align}
&\Gamma_{[k+1]}((k+1)\tau-\Delta)
\geq \frac{
\Gamma_{[k+1]}((k+2)\tau-\Delta)
}{1-\tau\eta\phi\Gamma_{[k+1]}((k+2)\tau-\Delta)}
\\&
\geq \Gamma_{[k+1]}((k+2)\tau-\Delta)
\geq \xi.
\end{align}
Therefore, to prove \eqref{goal4}, it is sufficient to show
\begin{align}
    \label{goal5}
\Gamma_{[k+1]}((k+1)\tau-\Delta)-\Gamma_{[k]}((k+1)\tau-\Delta)
    \leq L\psi(\alpha,k).
    \end{align}   
    Indeed,
\begin{align}
    \label{goal5}
&\Gamma_{[k+1]}((k+1)\tau-\Delta)-\Gamma_{[k]}((k+1)\tau-\Delta)
\\&
=
F(\mathbf w((k+1)\tau-\Delta))-F(\boldsymbol{c}_k((k+1)\tau-\Delta))
\\&
    \leq 
    L\Vert\mathbf w((k+1)\tau-\Delta)-\boldsymbol{c}_k((k+1)\tau-\Delta)\Vert,
    \end{align}   
    so that the result directly follows from Proposition \ref{prop2}.
    The Proposition is thus proved.
\end{proof}

\subsection{Proof of Theorem \ref{thm1}}
\begin{theorem} \label{thm1}
If $F_i(\cdot)$ is convex, $L$-Lipschitz and
$\beta$-smooth, when $\eta < \frac{2}{\beta}$, 
\begin{align} \label{78}
        &F(\boldsymbol{w^K})-F(\boldsymbol{w^*})\\
        &\leq \frac{1}{2\eta\phi T}+\sqrt{\frac{1}{4\eta^2\phi^2 T^2}+ \frac{L\Psi(\alpha)}{\eta\phi T}}+L\psi(\alpha,K).
\end{align}
where $\Psi(\alpha)=\sum\limits_{k=1}^{K}\psi(\alpha,k)$.\\
\end{theorem}

\begin{proof} 
To derive (35), consider $\eta \leq \frac{1}{\beta}$ and let $\xi^* >0$ be defined such that $T\eta\phi-\frac{L\Psi(\alpha)}{\xi^{*2}}>0$ and
\begin{align}
    \xi^* = \frac{1}{T\eta\phi-\frac{L\Psi(\alpha)}{\xi^{*2}}}.
\end{align}
Solving, we obtain
\begin{align}
    \xi^* = \frac{1}{2\eta\phi T} 
    +\sqrt{\frac{1}{4\eta^2\phi^2 T^2} + \frac{L\Psi(\alpha)}{\eta\phi T}}
\end{align}
(which indeed satisfies $T\eta\phi-\frac{L\Psi(\alpha)}{\xi^{*2}}>0$).
Now, let $\xi>\xi^*$, and assume that, under such $\xi$, the conditions of Proposition \ref{prop2} are all satisfied. Then, it follows that
    \begin{align}
        F(\mathbf w((K+1)\tau-\Delta))-F(\boldsymbol{w^*})<
        \frac{1}{T\eta\phi-\frac{L\Psi(\alpha)}{\xi^2}}
        \leq \frac{1}{T\eta\phi-\frac{L\Psi(\alpha)}{\xi^{*2}}}=\xi^*< \xi.
    \end{align}
    In other words, this shows a contradiction with condition (4) of Proposition \ref{prop2}.
    Therefore, at least
    one of the conditions of Proposition cannot be satisfied, for any $\xi>\xi^*$.
    Conditions (1) and (2) are clearly satisfied since $\eta\leq 1/\beta$   and
    $$\xi>\xi^*=\frac{1}{T\eta\phi-\frac{L\Psi(\alpha)}{\xi^{*2}}}>0.$$
    Therefore, either conditions (3) or (4) are violated, implying that
\begin{align}
\label{contr}
\min \{ F(\mathbf w((K+1)\tau-\Delta))
    ,
    \min_kF(\boldsymbol{c}_k((k+1)\tau-\Delta))\}-F(\boldsymbol{w^*})\leq\xi^*.
\end{align}
    Using Proposition \ref{prop1}, we have that
\begin{align}
&    F(\mathbf w((k+1)\tau-\Delta))
    \leq
    F(\boldsymbol{c}_k((k+1)\tau-\Delta))
    \\&
+ |F(\mathbf w((k+1)\tau-\Delta))-F(\boldsymbol{c}_k((k+1)\tau-\Delta))|
\\&
\leq
    F(\boldsymbol{c}_k((k+1)\tau-\Delta))
    \\&
+ L\Vert \mathbf w((k+1)\tau-\Delta)-\boldsymbol{c}_k((k+1)\tau-\Delta)\Vert
\\&
\leq
F(\boldsymbol{c}_k((k+1)\tau-\Delta))+L \psi(\alpha,k)
\\&
\leq
F(\boldsymbol{c}_k((k+1)\tau-\Delta))+L \psi(\alpha,K),
\end{align}
    ($\psi(\alpha,k)$ is increasing in $k$)
    so that
    $$
    \min_kF(\boldsymbol{c}_k((k+1)\tau-\Delta))
    \geq\min_k\{F(\mathbf w((k+1)\tau-\Delta))-L \psi(\alpha,K),\}    
    $$
    and \eqref{contr} implies
    \begin{align}
    \min_{k\leq K}\{F(\mathbf w((k+1)\tau-\Delta))\}-L \psi(\alpha,K)    
-F(\boldsymbol{w^*})\leq\xi^*.
\end{align}
The result of the theorem then directly follows.

\end{proof}

\end{document}